\documentclass{article}
\usepackage{amsthm}
     \PassOptionsToPackage{numbers, compress}{natbib}

\usepackage[preprint]{neurips_2026}

\usepackage[utf8]{inputenc} 
\usepackage[T1]{fontenc}    
\usepackage{hyperref}       
\usepackage{url}            
\usepackage{booktabs}       
\usepackage{amsfonts}       
\usepackage{nicefrac}       
\usepackage{microtype}      
\usepackage{xcolor}         
\usepackage{caption}
\usepackage{amsmath,amssymb,amsfonts}
\usepackage{mathtools}
\usepackage{bm}
\usepackage{physics}
\usepackage{graphicx}
\usepackage{subfigure}
\usepackage{algorithm}
\usepackage{algpseudocode}

\newtheorem{proposition}{Proposition}
\newtheorem{corollary}{Corollary}

\title{Entropic Auto-Encoding via Implicit Free-Energy Minimization}

\author{%
  Hazhir Aliahmadi$^{1}$\thanks{These authors contributed equally to this work.} \And
  Irina Babayan$^{1}$\footnotemark[1] \And
  Greg van Anders$^{1}$ \\
  $^{1}$Department of Physics, Engineering Physics \& Astronomy,
  Queen's University, Kingston, ON, Canada \\
  \texttt{\{hazhir.aliahmadi, babayan.irina, gva\}@queensu.ca}
}

\begin{document}

\maketitle

\begin{abstract}
  Despite their ubiquity, variational autoencoders (VAEs) inherently suffer from posterior collapse, a failure mode in which latent variables are effectively ignored. This failure arises because explicit prior imposition drives optimization toward loss landscape regions corresponding to uninformative latent representations. Here, we introduce Entropic Autoencoders (EAEs), a framework in which reconstruction loss is the only explicit objective, and entropy generates the latent variables' prior implicitly through a free energy-minimizing ensemble of encoders. This ensemble biases learning toward high-volume regions of near-optimal solutions, while decoder updates direct the search trajectories toward informative latent representations. We demonstrate that EAEs mitigate posterior collapse by learning non-Gaussian, multimodal latent distributions that yield diverse, data-consistent generations and preserve different forms of underlying structure in the data. As a proof-of-concept, we show that an EAE captures a superposition of the known low-dimensional dynamics of a reaction-diffusion process. Then, we show that an EAE identifies implicit categorical distinctions in MNIST latent representations, and displays a hierarchical understanding of facial structure on the CelebA dataset, from an ``all-human'' face to individual-dependent features.
\end{abstract}

\section{Introduction}

Dimensionality reduction has evolved from linear variance-maximizing projections such as principal component analysis (PCA) \citep{hotellingAnalysis1933} to nonlinear encoder--decoder architectures such as autoencoders \citep{hintonReducingDimensionalityData2006}, and even further to probabilistic latent-variable models such as variational autoencoders (VAEs) \citep{kingmaAutoEncoding2022b}, wherein representations are learned by optimizing variational objectives for otherwise intractable posteriors. This added flexibility, however, introduces posterior collapse: a failure mode in which the variational posterior becomes indistinguishable from the prior, the decoder ignores the latent variables, and the learned representation becomes uninformative \citep{wangPosterior2021,lucasDon2019,bowmanGenerating2016a,chenVariational2016,oordNeural2017,oordConditional2016}. 

Posterior collapse is often attributed either to Kullback-Leibler divergence (KL) pressure, which biases the variational posterior toward the prior \citep{bowmanGenerating2016a,kingmaImproved2016,sonderbyLadder2016}, or to expressive decoders, which can fit the data while ignoring the latent variables \citep{semeniutaHybrid2017,petitPreventing2021,chenVariational2016}. More broadly, both mechanisms modify the optimization landscape. Under ELBO optimization, such landscape effects can stabilize low-loss basins that favor reconstruction over inference \citep{zhaoInfoVAE2019,songScaleVAE2024}, reinforcing ill-conditioned stationary points \citep{lucasDon2019} and latent non-identifiability \citep{wangPosterior2021}, and producing  ``accurate'' decoders that are insensitive to latent codes.

In non-convex learning, Bayesian formulations can be viewed as loss basin selection, with posterior parameter concentration shaped both by objective values and by the distribution of prior mass across parameter space \citep{mackayBayesian1992, nealBayesianLearningNeural1996}. Prior choice is thus central to these formulations, but generic prior specification is difficult because informative latent structure is typically problem-dependent. A classical guiding principle is maximum entropy \citep{jaynes1}, which selects the least-committal distribution consistent with specified constraints and symmetries, separating explicit modeling assumptions from incidental structure \citep{jaynesPrior1968a}. Recent statistical physics-inspired approaches extend this strategy further by making the “effective prior” emerge implicitly from landscape geometry and collective variable dynamics, rather than being hand-specified \citep{babayanSufficient2025}.

The EAE approach we present here lets effective priors emerge from a canonical (Gibbs) ensemble over encoder parameters at finite temperature. This ensemble enables implicit free energy minimization, where the reconstruction loss defines the energy and the ensemble captures the effect of entropy. Then, the decoder guides the ensemble-induced distribution toward high-volume and reconstruction-relevant regions of the loss landscape, avoiding posterior collapse without explicit KL or entropy regularization.

\paragraph{Contributions.}
The main contributions of this paper are as follows:

\begin{itemize}
    \item We introduce \emph{Entropic Autoencoders (EAEs)}, a framework in which the \emph{only explicit objective} is reconstruction loss, while entropy affects training implicitly through an ensemble of encoder configurations. We provide a theoretical interpretation of EAEs as an implicit free-energy--based framework where the encoder ensemble induces a latent distribution, and the decoder is then updated by minimizing reconstruction error in expectation over this distribution, avoiding posterior collapse by selecting reconstruction-consistent latent representations that are supported by many encoder configurations.

    \item We show that EAEs recover intrinsically low-dimensional and dynamically meaningful representations, rather than merely compressing the data. In a reaction–diffusion system with known structure, the learned latent variables capture a superposition of admissible dynamics, demonstrating that the representation reflects the underlying generative mechanisms of the data.

    \item We empirically demonstrate that EAEs mitigate posterior collapse while preserving high-quality reconstructions and generative performance. In contrast to VAEs with Gaussian priors, which typically yield unimodal and often collapsed latent distributions, EAE learns non-Gaussian, multimodal latent representations, leading to diverse and data-consistent generations on standard benchmark datasets, including MNIST and Frey Faces.

    \item On the CelebA dataset, we analyze EAE dynamics across basins of the optimization landscape. We show that non-informative basins produce a generic \emph{all-human} face— an identity-agnostic prototype that captures population-level facial structure— while near-optimal ensemble exploration drives the system toward informative basins that generate diverse and realistic faces consistent with the data distribution. These generated-image features indicate that the latent space captures a reduced-dimensional representation of the dataset, which preserves the essential structure required for meaningful generation.
  
\end{itemize}
\section{Background and Related Work \label{sec:bg}}

\paragraph{Notation.}
Let $Y=\{y_n\}_{n=1}^N$ denote a dataset of observations with $y_n \in \mathbb{R}^{n_y}$. We associate to each observation a latent variable $z_n \in \mathbb{R}^{n_z}$ obtained through an encoder map $E_\phi:\mathbb{R}^{n_y}\to\mathbb{R}^{n_z}$, parameterized by $\phi \in \mathbb{R}^{n_\phi}$, such that $z = E_\phi(y)$. Reconstructions are produced by a decoder map $D_\vartheta:\mathbb{R}^{n_z}\to\mathbb{R}^{n_y}$, parameterized by $\vartheta \in \mathbb{R}^{n_\vartheta}$, yielding $\hat{y} = D_\vartheta(z)$. 

\paragraph{From deterministic to variational autoencoders.}
A deterministic autoencoder \citep{hintonReducingDimensionalityData2006} learns an encoder--decoder pair $(E_\phi,D_\vartheta)$ by minimizing the reconstruction objective $L_{\mathrm{rec}}(\phi,\vartheta):=\sum_{n=1}^N \ell\!\left(y_n,D_\vartheta(E_\phi(y_n))\right)$, where $\ell(y,\hat y)\ge 0$ is a pointwise reconstruction loss, e.g., $\ell(y,\hat y)=\tfrac12\|y-\hat y\|^2$, and $z_n=E_\phi(y_n)$ is a point estimate of the latent representation. This formulation provides no mechanism to model uncertainty or to generate samples beyond the training data. A probabilistic extension replaces the deterministic code with an approximate posterior $q_\phi(z|y)$, introduces a prior $p(z)$ over latent variables, and interprets the decoder as a likelihood model $p_\vartheta(y|z)$. The reconstruction term then becomes an expected negative log-likelihood, $\mathbb{E}_{q_\phi(z|y_n)}[-\log p_\vartheta(y_n|z)]$, while the latent distribution is regularized by constraining its KL divergence from the prior, $D_{\mathrm{KL}}(q_\phi(z|y_n)\|p(z))\le \epsilon$. Writing the corresponding Lagrangian gives the weighted variational loss (see \citep{higginsBetaVAE2017} for details)
\[
L(\phi,\vartheta;\lambda)
=
\sum_{n=1}^N
\mathbb{E}_{q_\phi(z|y_n)}
\!\left[-\log p_\vartheta(y_n|z)\right]
+
\lambda
\sum_{n=1}^N
D_{\mathrm{KL}}\!\left(q_\phi(z|y_n)\|p(z)\right),
\qquad
\lambda \ge 0 .
\]
where $\lambda$ is the nonnegative Lagrange multiplier controlling the strength of prior matching. The case $\lambda=1$ recovers the standard VAE loss \citep{kingmaAutoEncoding2022b}, while $\lambda\neq1$ can improve representations by altering the reconstruction--regularization trade-off \citep{higginsBetaVAE2017}.

\paragraph{Posterior collapse.}
A well-known failure mode of VAEs $(\lambda = 1)$ is \emph{posterior collapse}, in which the encoder distribution becomes indistinguishable from the prior, $q_{\phi}(z\mid y_n)\approx p(z)$ for all $y_n\in Y$. In this regime, the latent variable $z$ becomes effectively independent of the input $y_n$ and captures little information about it. Equivalently, the KL term in the ELBO vanishes, and the model relies almost entirely on the decoder to reconstruct the data. Ref. \citep{alemiFixingBrokenELBO2018} showed that reducing the weight of the KL term, $\lambda<1$, can mitigate posterior collapse. More generally, Ref. \citep{wenzelHowGoodBayes2020} interpret the Lagrange multiplier as a \emph{temperature}, $\lambda := T$, and demonstrate that predictive performance can improve under a cold posterior ($T<1$). This perspective is consistent with approaches such as KL annealing \citep{bowmanGenerating2016a,fuCyclical2019,shaoControlVAEControllableVariational2020}, which dynamically control the influence of the KL term during training to reduce the posterior collapse. 

Another line of work attributes posterior collapse primarily to the choice of prior under a fixed training temperature, usually $T=1$, and aims to modify the prior to better match a given fixed-temperature regime. This includes employing more flexible explicit priors, such as hyperspherical \citep{davidsonHyperspherical2022}, uniform \citep{zhaoUnsupervised2018,oordNeural2017}, and mixture priors \citep{chienAmortizedMixturePrior2020a,klushynLearning2019}. 
A related strategy leaves the explicit prior unchanged and instead regularizes the objective, reshaping the optimization landscape so that informative latent representations are favored under the chosen prior and temperature \citep{zhengUnderstanding2019,maMAE2018,zhaoInfoVAE2019}.
By contrast, Ref.~\citep{babayanSufficient2025} recently showed that an effective prior over collective variables can emerge implicitly, without explicit regularization, from a canonical ensemble over model parameters via fixed-temperature free-energy minimization. The idea of an emergent effective prior highlights a duality between prior choice and temperature: one may either fix the prior (e.g., Gaussian) and adjust the temperature, or fix the temperature (e.g., $T=1$) and modify the prior. Following this insight, EAEs perform iterative fixed-temperature free-energy minimization using reconstruction loss as the only explicit objective, with regularization arising implicitly from the canonical ensemble of encoders.

From a loss-landscape geometry perspective, well-generalizing solutions often lie in broad, flat regions of the landscape rather than in sharp minima \citep{FantasticGeneralizationMeasures2019,chaudhariEntropySGDBiasingGradient2019c,petzkaRelativeFlatnessGeneralization2021}, but posterior collapse indicates the existence of degenerate, uninformative loss-landscape flatness, where the likelihood becomes insensitive to some latent directions in an ill-conditioned region and renders those variables effectively non-identifiable \citep{lucasDon2019,wangPosterior2021}. More specifically, collapsed regions of encoder-parameter space may be broad because many resulting latent variable outcomes are equally irrelevant, rather than because they encode robust, useful structure. Architectural remedies to this latent variable insensitivity weaken or restructure the decoder such that reconstruction must rely on latent variables, favoring informative rather than degenerate flatness \citep{semeniutaHybrid2017,petitPreventing2021,chenVariational2016}. EAEs address this issue differently by using the entropy of encoder configurations to induce a prior over latent variables. The decoder is iteratively trained on these ensemble-generated latent codes, discouraging collapse by requiring the decoder to depend on the ensemble of codes to determine the reconstruction.
\section{Entropic Autoencoders \label{sec:eae}}

\paragraph{Overview.}
Entropic Autoencoders (EAEs) replace single-encoder training with a
finite-temperature ensemble of encoders, keeping the reconstruction loss
$L_{\mathrm{rec}}(\phi,\vartheta)$ as the only explicit objective. At outer
iteration $k$, the decoder parameters $\vartheta_k$ are fixed over inner loop $i=1,...,M$, during which Simmering
samples encoder parameters from the reconstruction-induced Gibbs measure
defined in Section~\ref{subsec:canonical_encoder_ensemble}. Section~\ref{subsec:cv_free_energy} shows how this ensemble induces
a collective-variable free energy whose entropy term favors high-volume regions
of encoder space. Then, the decoder is updated using reconstruction gradients
averaged over this ensemble-induced latent distribution (sampled during the inner loop), which corresponds to
descent along the decoder free energy landscape as described in
Section~\ref{subsec:decoder_entropic_selection}. Repeating these coupled steps
biases learning toward latent representations that are both reconstructive and
supported by many encoder configurations; see Algorithm~\ref{alg:eae_training} for a schematic description fo the training process.

\begin{algorithm}[t]
\caption{Entropic Autoencoder training}
\label{alg:eae_training}
\begin{algorithmic}[1]
\Require data $Y=\{y_n\}_{n=1}^N$, initial parameters $(\phi_0,\vartheta_0)$
\Require inverse temperature $\beta$, ensemble size $M$, minibatch size $N_b$
\Require tolerance $\varepsilon$, decoder update rule $\mathcal U_\vartheta$
\State Define minibatch loss
$\widehat L_{\mathcal B}(\phi,\vartheta)
:=
|\mathcal B|^{-1}\sum_{n\in\mathcal B}
\ell\!\left(y_n,D_{\vartheta}(E_\phi(y_n))\right)$
\State $k\gets0$, $\phi\gets\phi_0$, $\bar g_\vartheta\gets\infty$
\While{$\|\bar g_\vartheta\|>\varepsilon$}
    \State Fix decoder $\vartheta_k$
    \Comment{target $p_{\vartheta_k}(\phi\mid Y)$ in Eq.~\eqref{eq:eae_encoder_gibbs}}
    \State $\mathcal E\gets\emptyset$, $\mathcal G_\vartheta\gets\emptyset$
    \For{$i=1,\ldots,M$}
        \State Draw minibatch $\mathcal B_i\subset\{1,\ldots,N\}$
        \State $\mathcal E\gets \mathcal E\cup\{\phi\}$
        \Comment{store encoder sample}
        \State $\mathcal G_\vartheta\gets
        \mathcal G_\vartheta\cup
        \left\{
        \nabla_\vartheta \widehat L_{\mathcal B_i}(\phi,\vartheta_k)
        \right\}$
        \Comment{store decoder-gradient sample}
        \State $\phi\gets
        \mathrm{SimmeringStep}\!\left(
        \phi,\nabla_\phi \widehat L_{\mathcal B_i}(\phi,\vartheta_k),\beta
        \right)$
        \Comment{Appendix~\ref{app:eae_simmering}}
    \EndFor
    \State $\bar g_\vartheta\gets M^{-1}\sum_{g\in\mathcal G_\vartheta} g$
    \Comment{estimate of Eq.~\eqref{eq:eae_decoder_free_energy_gradient}}
    \State $\vartheta_{k+1}\gets \mathcal U_\vartheta(\vartheta_k,\bar g_\vartheta)$
    \Comment{one first-order decoder
update}
    \State $k\gets k+1$
\EndWhile
\State \Return decoder $\vartheta_k$ and encoder ensemble $\mathcal E$
\end{algorithmic}
\end{algorithm}

\subsection{Canonical encoder ensemble and induced latent distribution}
\label{subsec:canonical_encoder_ensemble}
For fixed decoder parameters $\vartheta_k$, EAE defines a canonical, i.e., maximum-entropy, ensemble over
encoder parameters,
\begin{equation}
p_{\vartheta_k}(\phi\mid Y)
=
\frac{1}{Z(\vartheta_k,\beta;Y)}
\exp\!\left\{
-\beta L_{\mathrm{rec}}(\phi,\vartheta_k)
\right\},
\qquad
Z(\vartheta_k,\beta;Y)
=
\int
\exp\!\left\{
-\beta L_{\mathrm{rec}}(\phi,\vartheta_k)
\right\}\,d\phi .
\label{eq:eae_encoder_gibbs}
\end{equation}
Here $\beta=1/T$ is the inverse temperature, and $Z(\vartheta_k,\beta;Y)$ is the
encoder partition function. Thus, under a fixed decoder, encoder configurations are
weighted by reconstruction energy: low-loss encoders are favored, but finite
temperature allows a family of near-optimal encoders to contribute to $p_{\vartheta_k}(\phi|Y)$.

In practice, we sample from $p_{\vartheta_k}(\phi\mid Y)$ using Simmering \citep{babayanSufficient2025}. For fixed
$\vartheta_k$, Simmering treats encoder parameters as finite-temperature,
one-dimensional, classical particles that time-evolve under the potential $L_{\mathrm{rec}}(\phi,\vartheta_k)$, with
forces $-\nabla_\phi L_{\mathrm{rec}}(\phi,\vartheta_k)$. A thermostat maintains
a chosen target temperature $T=1/\beta$ so that, after equilibration, the trajectory samples
near-optimal encoders rather than converging to a single minimizer; details are
given in Appendix~\ref{app:eae_simmering}.

For each observation $y_n$, the sampled encoders induce latent samples
$z_n^{(i)}=E_{\phi^{(i)}}(y_n)$. Equivalently, the latent distribution is the marginal distribution induced by the encoder Gibbs
measure through the encoder map,
\begin{equation}
p_k(z\mid y_n)
=
\int_{\mathbb R^{n_\phi}}
\delta\!\left(z-E_{\phi}(y_n)\right)
p_{\vartheta_k}(\phi\mid Y)\,d\phi .
\label{eq:eae_latent_pushforward}
\end{equation}
Therefore, uncertainty in the latent representation arises from variability
across encoder configurations, rather than from an explicitly prescribed
variational posterior or a KL penalty.
At evaluation time, the trained decoder and sampled encoder ensemble can be used
to construct the empirical latent ensemble by applying
$z_n^{(i)}=E_{\phi^{(i)}}(y_n)$; the sampling procedure is given in
Appendix~\ref{app:eae_sampling}.
\subsection{Collective-variable free energy and entropic bias}
\label{subsec:cv_free_energy}

The collective behaviour of the canonical ensemble of encoders is characterized by collective variables $\boldsymbol{\theta}=\boldsymbol{\theta}(\phi,Y)\in\mathbb R^{n_\theta}$. These
collective variables need not be known or prescribed during training: Simmering
samples $\phi$, and the corresponding values of $\boldsymbol{\theta}$ are induced
implicitly by the sampled encoders and the data.

At each temperature $T=1/\beta$, marginalizing the encoder Gibbs measure induces
a density over collective variables,
\begin{equation}
p_\beta(\boldsymbol{\theta}\mid \vartheta_k,Y)
=
\int
\delta^{n_\theta}\!\left(
\boldsymbol{\theta}(\phi,Y)-\boldsymbol{\theta}
\right)
p_{\vartheta_k}(\phi\mid Y)\,d\phi
\propto
\exp\!\left\{
-\beta F_\beta(\boldsymbol{\theta},\vartheta_k,Y)
\right\}.
\label{eq:eae_theta_pushforward}
\end{equation}
As derived in Appendix~\ref{app:cv_free_energy}, the corresponding free energy
can be written, under a first-cumulant, or slow-variation, approximation on
iso-$\boldsymbol{\theta}$ sets, as
\begin{equation}
F_\beta(\boldsymbol{\theta},\vartheta_k,Y)
=
\big\langle L_{\mathrm{rec}}(\phi,\vartheta_k)\big\rangle_{\boldsymbol{\theta}}
-
\frac{1}{\beta}S(\boldsymbol{\theta},Y),
\qquad
S(\boldsymbol{\theta},Y)=\log\Omega(\boldsymbol{\theta},Y),
\label{eq:eae_cv_free_energy}
\end{equation}
where $\Omega(\boldsymbol{\theta},Y)$ is the encoder-parameter volume, acting as an implicit prior over collective variables: values of
$\boldsymbol{\theta}$ supported by many encoder configurations receive greater
weight without specifying an explicit prior distribution. Thus, the
finite-temperature encoder ensemble defines an effective free-energy landscape
over collective variables in which reconstruction fidelity is balanced against
encoder-space volume.

This decomposition makes the entropic bias explicit. Since
$\Omega(\boldsymbol{\theta},Y)$ depends only on the ensemble-generated collective variables
$\boldsymbol{\theta}(\phi,Y)$ and the data, the entropy term is independent of the
decoder parameters, i.e., $\nabla_\vartheta S(\boldsymbol{\theta},Y)=0$. Therefore,
decoder updates reshape the collective-variable free energy only through the
conditional reconstruction term. For an infinitesimal decoder change
$\vartheta_{k+1}=\vartheta_k+\delta\vartheta$,
\begin{equation}
\delta F_\beta(\boldsymbol{\theta})
\approx
\left\langle
\nabla_{\vartheta}L_{\mathrm{rec}}(\phi,\vartheta_k)
\right\rangle_{\boldsymbol{\theta}}
\cdot \delta\vartheta .
\label{eq:eae_delta_cv_free_energy}
\end{equation}
Consequently, after each decoder update and resampling step, probability mass
shifts toward collective-variable induced latent features whose conditional reconstruction energy
has decreased, while the entropy term maintains a fixed geometric preference for
high-volume encoder features. Thus, the collective-variable free energy explicitly depicts how EAE balances reconstruction fidelity against encoder-space volume.

\subsection{Decoder updates as entropic selection}
\label{subsec:decoder_entropic_selection}

The encoder ensemble defines a decoder-level free energy through the same
partition function as in Eq. \ref{eq:eae_encoder_gibbs},
\begin{equation}
F_{\mathrm{dec}}(\vartheta;Y)
:=
-\frac{1}{\beta}
\log Z(\vartheta,\beta;Y),
\qquad
Z(\vartheta,\beta;Y)
=
\int
\exp\!\left\{-\beta L_{\mathrm{rec}}(\phi,\vartheta)\right\}\,d\phi .
\label{eq:eae_decoder_free_energy}
\end{equation}
Minimizing $F_{\mathrm{dec}}$ is therefore not equivalent to selecting a single
best encoder. Instead, a decoder has low free energy when a large volume of
encoder parameters achieves low reconstruction error under that decoder. Thus, this
decoder learning strategy favors decoders whose performance is supported by broad
families of good encoders.

As shown in Appendix~\ref{app:decoder_free_energy_gradient}, differentiating
$F_{\mathrm{dec}}$ gives
\begin{equation}
\nabla_\vartheta F_{\mathrm{dec}}(\vartheta_k;Y)
=
\mathbb E_{p_{\vartheta_k}(\phi\mid Y)}
\!\left[
\nabla_\vartheta L_{\mathrm{rec}}(\phi,\vartheta_k)
\right].
\label{eq:eae_decoder_free_energy_gradient}
\end{equation}
Hence, updating the decoder using the ensemble-averaged decoder gradient is
equivalent to decreasing the decoder free energy. After the decoder is updated,
Simmering samples a new encoder Gibbs measure determined by
$L_{\mathrm{rec}}(\phi,\vartheta_{k+1})$, which induces a new latent distribution.

At convergence, the decoder is self-consistent: it is stationary for the
reconstruction gradient averaged over the latent distribution induced by its own
stabilized encoder ensemble. In the limit $\beta\to\infty$, the encoder Gibbs
measure concentrates on reconstruction-loss minimizers, recovering the
deterministic autoencoder limit; at finite intermediate temperature, the ensemble
remains extended and supports nontrivial latent distributions, whereas at excessively high temperature, reconstruction fidelity weakly constrains the
encoder ensemble, yielding broad and low-information latent codes.

In the collective-variable view, this decoder update strategy changes the reconstruction-energy
landscape over the encoder-induced collective variables, while the entropy term
continues to favor high-volume encoder features. Thus, decoder learning performs
an entropically biased selection indirectly: the next encoder ensemble is drawn
toward latent features that are both supported by many encoder configurations and
useful for reconstruction under the updated decoder. This mitigates posterior
collapse by favoring high-volume latent structure that the decoder must use for reconstruction,
rather than identifying latent-variable directions it can ignore.

\section{Experiments \label{sec:experiments}}

To demonstrate that EAEs mitigate posterior collapse by learning meaningful latent representations, we present three training experiments aimed to show that an EAE can learn complex latent distributions \emph{and} reflect both explicit and implicit underlying structure in data. In Section \ref{sec:exp_1}, we show that an EAE can explicitly identify a superposition of admissible forms for the latent dynamics of a reaction-diffusion process with known low-dimensional structure.  In Sections \ref{sec:exp_2}--\ref{sec:exp_3}, we relax the explicit latent structure imposed in Section \ref{sec:exp_1} and train EAEs on image generation, a task with no known low-dimensional form. In these experiments, we find that EAEs encode dataset-level implicit structures in their latent representations. We include a discussion of practical considerations for EAE training implementations in Appendix \ref{app:limits_futurework}.  

\subsection{Recovery of meaningful low-dimensional representations with an EAE \label{sec:exp_1}}

\begin{figure}
    \centering
    \includegraphics[width=\linewidth]{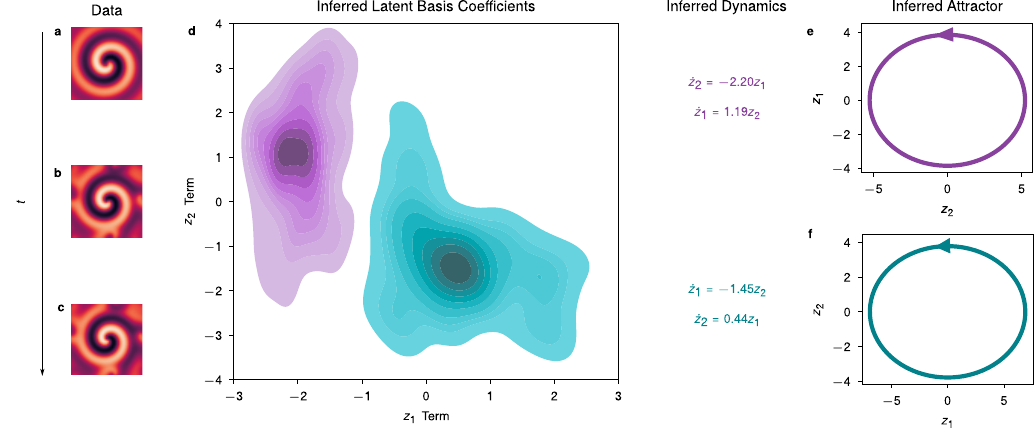}
    \caption{An EAE recovers a superposition of a system's known low-dimensional form in its latent dynamics (purple and teal regions in panel d, corresponding to the identified attractors in panels e and f respectively) from simulated data of a lambda-omega reaction-diffusion process \citep{championDatadriven2019} (initial, intermediate and final time snapshots of time-evolving data shown in panels a--c respectively). The EAE identifies these latent variable dynamics by using the ensemble-encoded latent variables to estimate the coefficients of candidate basis functions (see Appendix \ref{sec:appendix_exp1}). Both modes in the $(z_1,z_2)$-coefficient space (panel d) produce equivalent limit cycle dynamics (panels e and f) under a permutation of latent-variable coordinates, i.e., $z_1 \longleftrightarrow  z_2$. Beyond recovering the expected low-dimensional model of the data, the EAE's can depict implicit underlying structure such as latent representation symmetry. }
    \label{fig:toy_model}
\end{figure}

Figure \ref{fig:toy_model} shows the latent-variable dynamics  (\ref{fig:toy_model}e-f) inferred by an EAE trained from data describing the time evolution of a lambda-omega reaction-diffusion process (Fig. \ref{fig:toy_model}a--c) From 100x100-element spatially-discretized fixed-time snapshots of a wave and its time derivative generated by numerical integration from Ref. \citep{championDatadriven2019}, we trained an EAE to reconstruct the training data (Fig.\ref{fig:toy_model}a--c), with the aim of identifying the system's known reduced-dimensional form: oscillations near a limit cycle \citep{murrayMathBio1990}. We used the ensemble of encoders to infer the coefficients of several candidate basis functions (see \ref{sec:appendix_exp1}) describing the dynamics of each latent dimension (Fig. \ref{fig:toy_model} Inferred Dynamics). We identified two strongly-correlated combinations of linear-basis coefficients (Fig. \ref{fig:toy_model}d, and see \ref{sec:appendix_exp1}) both of which produce the anticipated limit cycle, i.e., attractor (Fig. \ref{fig:toy_model}e--f), in the learned low-dimensional space. Furthermore, the two coefficient combinations produce near-equivalent latent dynamics (Fig. \ref{fig:toy_model}e--f) under a permutation of the two latent dimensions, i.e., $z_1 \longleftrightarrow z_2$. Since the ``coordinate system'' of the latent space is not explicitly enforced during coefficient inference, the identified superposition of basis coefficients indicates that the EAE not only accurately models the reduced-dimensional dynamics, but also encodes the implicit structure of these dynamics, i.e., the symmetry in the choice of parameterization. Thus, Fig. \ref{fig:toy_model} shows that an EAE can recover accurate explicit forms for the underlying low-dimensional structure and provide insight into lower-dimensional model structure from high-dimensional data using only the reconstruction loss and its encoder ensemble.

\subsection{Comparison of latent variable distributions in image reconstruction tasks\label{sec:exp_2}}

\begin{table}[ht]
    \centering
    \caption{Autoencoder performance is assessed using reconstruction accuracy, via per-pixel mean-squared error on test data, and proportion of active latent units (see Appendix \ref{sec:appendix_exp2}), of three models (the EAE, the VAE, and the vanilla AE) on the MNIST \citep{lecunMNIST2010} and Frey Faces \citep{FreyFace} datasets. While the vanilla AE yields the lowest reconstruction error, the EAE outperforms both models in learning active latent representations.  }
    \begin{tabular}{|c|c|c|c|c|}
        \hline
         & \multicolumn{2}{c}{MNIST} & \multicolumn{2}{|c|}{Frey Faces} \\
        \cline{2-5}
         & Reconstruction Err. & Prop. of Active  & Reconstruction Err. & Prop. of Active  \\
         & (MSE on Test Data) & Latent Units & (MSE on Test Data) & Latent Units \\
         \hline
         EAE & 0.026 & 64/64 & 0.0052 & 64/64 \\
         VAE & 0.016 & 16/64 & 0.0046 & 3/64 \\
         Vanilla AE & 0.010 & 53/64 & 0.0014 & 64/64 \\
         \hline
    \end{tabular}
    \label{tab:comparison}
\end{table}

\begin{figure}
    \centering
    \includegraphics[width=\linewidth]{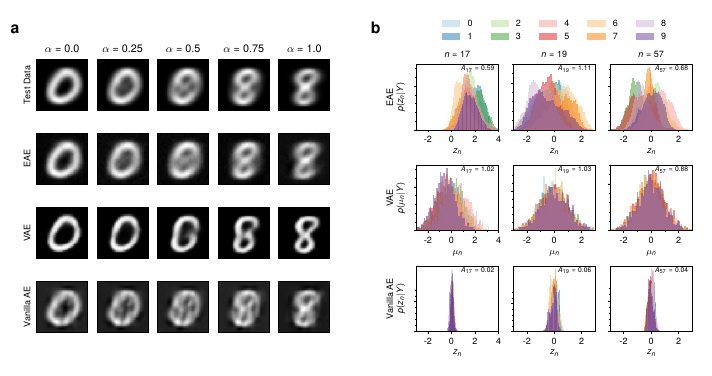}
    \caption{Analysis of interpolation capability (panel a) and learned latent distributions (panel b) of three autoencoder models (overall performance reflected in Table \ref{tab:comparison}) shows that the EAE produces the most data-consistent interpolation between average encodings of MNIST digits (panel a), and learns distinctly diverse and digit-category distinguishing distributions (panel b, top row) across sampled latent dimension distributions (more shown in Appendix \ref{sec:appendix_exp2} Fig. \ref{fig:appendix_mnist_vis}). EAE-generated images (panel a, second row) reflect the intrinsic variance of the corresponding digit data (panel a, top row), whereas the VAE (panel a, third row) produces artificially sharp interpolations, and the vanilla AE (panel a, bottom row) cannot interpolate altogether. Sampling digit-separated (legend at the top of panel b) latent distributions for three ``active'' dimensions (see Appendix \ref{sec:appendix_exp1}, activity measure $A_n$, in the corner of each subplot in panel b) in the latent space of each MNIST-trained model, we see that the VAE latent variable mean distributions exhibit significant overlap, while the vanilla AE distributions overlap and are relatively have low variance. In contrast, the EAE learns distinguishable distributions for different digits, reflecting the implicit categorical structure of the data in its latent representation. }
    \label{fig:mnist_compar}
\end{figure}

In this set of experiments, we compare EAE reconstruction performance and learned latent space distributions to other autoencoders in image reconstruction tasks, where, instead of explicit structure, the data may possess implicit structure, e.g., discrete categorical information in MNIST. 

Table \ref{tab:comparison} shows that, though a vanilla autoencoder (vanilla AE) can achieve a lower reconstruction error than the EAE and VAE, only an EAE can make complete use of its latent space dimensionality in two image reconstruction tasks, whereas vanilla autoencoders (vanilla AEs) exhibit inconsistent dimension employment and VAEs exhibit severe posterior collapse, with at most 25\% of its latent variables exceeding the latent activity threshold \citep{burdaImportance2015}. Though the vanilla AE achieves the lowest reconstruction error, this alone may not be an effective metric for assessing the models' latent representations, since in EAEs and VAEs, learning is implicitly or explicitly influenced by other (entropic or prior-matching) regularizing effects. However, the characterization of latent variable activity shows that, in for VAEs, the aforementioned regularization effects fail to promote optimal latent variable engagement, whereas the ensemble-induced entropic bias produces high-variance latent representation in the EAE.

To further characterize the learned latent representation differences between the three autoencoders, we compared their interpolations between digits $0$ and $8$ (Fig. \ref{fig:mnist_compar}a), with digit encodings linearly interpolated by a parameter $\alpha$ (see Appendix \ref{sec:appendix_exp2}), and their learned latent distributions separated by digit category (Fig. \ref{fig:mnist_compar}b), with more latent dimension comparisons shown in Fig. \ref{fig:appendix_mnist_vis}. We chose to analyze the learned representations for MNIST rather than for the Frey Faces dataset as the MNIST data offers an opportunity to look for implicit structure (digit categories) in the latent space encodings. Comparing the interpolations of each model with the intrinsic variance of the data (Fig. \ref{fig:mnist_compar}a, top row vs. other rows), we find that, though the EAE and vanilla AE both had a high proportion of latent activity (Table \ref{tab:comparison}) on this dataset, the EAE uses its high-dimensional representations of the digits to recover nuances in intra-category variability (EAE vs. Test Data in \ref{fig:mnist_compar}a), whereas the vanilla AE cannot interpolate smoothly and its generations exhibit data-inconsistent artifacts. While the VAE's interpolations are contrast relative to the EAE, when they are compared to the data (VAE vs. Test Data in \ref{fig:mnist_compar}a), we see that this interpolation sharpness is another data-inconsistent artifact indicating a mismatch between identified latent structure and the data distribution. Taken altogether, Table \ref{tab:comparison} and Fig. \ref{fig:mnist_compar}a show that, though the EAE exhibits the highest reconstruction error, its representation flexibility allows it to maximally take advantage of its latent space dimensionality to capture data distribution nuance that is easily missed under posterior collapse in, e.g., a VAE.

Figure \ref{fig:mnist_compar}b brings the difference in learned representations between the EAE and other autoencoders into sharp relief, where the EAE uniquely encodes a clear distinction between digit categories in its learned latent distribution. In Fig. \ref{fig:mnist_compar}b, we see that, though both the vanilla AE and the EAE had high proportions of latent variable activity on this dataset (Table \ref{tab:comparison}), the EAE learns more complex and higher-variance distributions for each dimension (Fig. \ref{fig:mnist_compar} top row), that take on distinct forms for different digit categories. In contrast, the Vanilla AE's latent distributions (Fig. \ref{fig:mnist_compar} bottom row) have relatively low variance and overlap over much of the encoding (x-axis) range. Furthermore, though the VAE's per-dimension latent activities ($A_n$ in subplots of Fig. \ref{fig:mnist_compar}) are comparable in magnitude to those of the EAE, its latent distributions for different digits completely overlap, indicating an absence of encoded distinction between the digit categories in the latent space. Fig. \ref{fig:mnist_compar}b (and Fig. \ref{fig:appendix_mnist_vis}) shows that, beyond mitigating posterior collapse overall, an EAE also learn richer and more informative representations on a per-dimension level, including encoding implicit latent structure in the learned distributional form.

\subsection{EAE Dynamics and Reconstruction Morphology\label{sec:exp_3}}

\begin{figure}
    \centering
    \includegraphics[width=\linewidth]{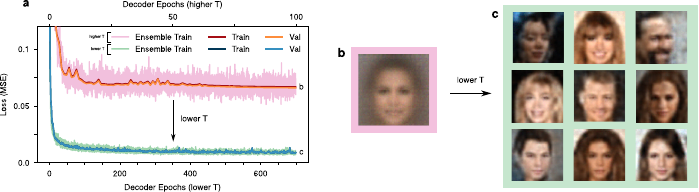}
    \caption{On the CelebA dataset \citep{celebA}, an EAE can exhibit two different generative morphologies -- an ``all-human'' face (panel b), and a variety of data-consistent image generations (panel c) (see Appendix \ref{sec:appendix_exp3}). The model corresponding to the ``all-human'' face (pink, red and orange training curves and upper x-axis, panel a) was trained at with a higher temperature parameter, resulting in a more generic but still data-faithful reconstruction. Training an EAE at a lower temperature (teal, navy and blue lines and lower x-axis, panel a) manifests individual-dependent facial morphology in image generation. Note that for both models, the lower-opacity ``Ensemble Train'' curves (pink and teal curves, panel a) show the sampled encoder ensemble member losses, and the opaque curves (red, orange, navy, blue curves, panel a) show the loss after a decoder parameter update. }
    \label{fig:celebA}
\end{figure}


For successful facial image reconstruction, the EAE must encode both the general properties of a face, and be able to manifest individual-level differences in reconstruction. Figure \ref{fig:celebA}b shows a reconstruction sample on test data after training an EAE (panel a, pink, red and orange curve) at relatively high temperature (for the given problem), which depicts an ``all-human'', generic facial structure but not unique individual-level facial characteristics. The relatively high variance of the higher-temperature ensemble member loss distribution (pink envelope vs. teal envelope, \ref{fig:celebA}a) indicates that the sampled region of the solution space is not strongly constrained by reconstruction-loss minimization, i.e., is relatively uninformative, but the model nonetheless learns latent structure that manifests the universal properties of faces in its reconstructions. After lowering the temperature (\ref{fig:celebA}a, teal, navy and blue curve), we observe the predicted effect of temperature on the decoder parameter search mechanism: at a suitably low temperature, the decoder updates direct latent structure learning toward informative representations that facilitate individual-level facial feature reconstructions (\ref{fig:celebA}c) while still adhering to the dataset-level structure present in a higher-temperature model. In addition to achieving accurate reconstruction, the ``emergence'' of more complex facial reconstructions at a lower temperature from a more generic reconstruction at a higher temperature implies that the EAE can form quite complex implicit models of data, including a hierarchical understanding of facial feature distinction.

\section{Discussion}

EAEs suggest a different interpretation of posterior collapse: the failure may
not lie only in optimization, but also in choosing an inappropriate
effective description of the data. From a statistical-physics
viewpoint, dimensional reduction is useful when the reduced variables behave like
appropriately defined collective variables. A VAE makes the latent distribution explicit
by specifying a variational family and matching it to a prescribed prior, most
often Gaussian. This gives a tractable distributional representation, but also
fixes an effective latent geometry in advance. In physical terms, this risks
constructing an effective theory in the wrong order parameter; when the imposed
latent description is misaligned with the structure needed for reconstruction,
the decoder can learn to ignore the latent code, producing posterior collapse.

EAEs keep the useful distributional aspect of VAEs, but reverse the order of
construction. Rather than prescribing the latent distribution, it samples a
finite-temperature ensemble of encoders and lets the latent structure emerge
from reconstruction loss and encoder-space volume. The decoder is then updated
against this ensemble, so latent features are reinforced only when they are both
high-volume in encoder space and useful for reconstruction. In this sense, EAEs
treat latent variables as emergent collective degrees of freedom. Temperature
facilitates exploration among near-equivalent encoders, entropy induces a
prior-like bias toward robust latent structure, and decoder learning pursues informative latent structures.

\section{Conclusion}

We proposed a new autoencoder training framework, Entropic Autoencoders (EAEs), which mitigates posterior collapse. EAEs induce, rather than impose, a prior over latent variables via a maximum-entropy ensemble of encoders. While a decoder employs the encoder ensemble for reconstruction, it guides the next encoder ensemble toward more informative inferred latent distributions preventing posterior collapse. We showed that EAEs both avoid posterior collapse in image generation tasks, and produce latent distributions which depict explicit (governing equations of reaction-diffusion dynamics) and implicit (categorical and morphological distinction) underlying structure in data. 

\newpage
\begin{ack}
We acknowledge the support of the Natural Sciences and Engineering Research Council of Canada (NSERC) grant RGPIN-2019-05655. IB acknowledges the support of an Ontario Graduate Scholarship.
\end{ack}

\bibliographystyle{plainnat}
\bibliography{NeuralPS_EAE}


\appendix

\section{Appendix / supplemental material}

\subsection{Collective-variable free energy of the encoder ensemble\label{app:cv_free_energy}}

In this appendix, we derive the collective-variable free-energy relation used in
Section~\ref{subsec:cv_free_energy}. Throughout, the decoder parameters
$\vartheta$ are fixed unless stated otherwise. The encoder ensemble is the Gibbs
measure
\[
p_{\vartheta}(\phi\mid Y)
=
\frac{1}{Z(\vartheta,\beta;Y)}
\exp\!\left\{-\beta L_{\mathrm{rec}}(\phi,\vartheta)\right\},
\]
with partition function
\[
Z(\vartheta,\beta;Y)
=
\int_{\mathbb R^{n_\phi}}
\exp\!\left\{-\beta L_{\mathrm{rec}}(\phi,\vartheta)\right\}
\,d\phi .
\]

Let
\[
\boldsymbol{\theta}
=
\boldsymbol{\theta}(\phi,Y)
\in \mathbb R^{n_\theta}
\]
be a decoder-independent set of collective variables that coarse-grain the
latent representation induced by encoder parameters $\phi$ on the dataset $Y$.
These variables need not be known, evaluated, or prescribed during training.
They provide an implicit coarse description of the structure sampled by
Simmering: the algorithm samples encoder parameters $\phi$, while any such
collective variables are induced only through the sampled encoders and the data.

\begin{proposition}[Collective-variable free energy]
\label{prop:cv_free_energy}
Marginalizing the encoder Gibbs measure over all encoder parameters that realize
the same collective variables gives the induced marginal density over
$\boldsymbol{\theta}$. Define the encoder-parameter volume associated with a
collective-variable value $\boldsymbol{\theta}$ by
\[
\Omega(\boldsymbol{\theta},Y)
=
\int_{\mathbb R^{n_\phi}}
\delta^{n_\theta}
\!\left(
\boldsymbol{\theta}(\phi,Y)-\boldsymbol{\theta}
\right)
\,d\phi ,
\]
and let
\[
S(\boldsymbol{\theta},Y)
=
\log \Omega(\boldsymbol{\theta},Y).
\]
For any integrable function $g$, define the conditional average over the
iso-$\boldsymbol{\theta}$ set by
\[
\langle g(\phi)\rangle_{\boldsymbol{\theta}}
=
\frac{1}{\Omega(\boldsymbol{\theta},Y)}
\int_{\mathbb R^{n_\phi}}
g(\phi)
\delta^{n_\theta}
\!\left(
\boldsymbol{\theta}(\phi,Y)-\boldsymbol{\theta}
\right)
\,d\phi .
\]
The collective-variable free energy is defined implicitly by
\[
\exp\!\left\{-\beta F_\beta(\boldsymbol{\theta},\vartheta,Y)\right\}
=
\int_{\mathbb R^{n_\phi}}
\exp\!\left\{-\beta L_{\mathrm{rec}}(\phi,\vartheta)\right\}
\delta^{n_\theta}
\!\left(
\boldsymbol{\theta}(\phi,Y)-\boldsymbol{\theta}
\right)
\,d\phi .
\]
Equivalently,
\[
F_\beta(\boldsymbol{\theta},\vartheta,Y)
=
-\frac{1}{\beta}
\log
\left\langle
\exp\!\left\{-\beta L_{\mathrm{rec}}(\phi,\vartheta)\right\}
\right\rangle_{\boldsymbol{\theta}}
-
\frac{1}{\beta}S(\boldsymbol{\theta},Y).
\]
Consequently, the induced density over collective variables is
\[
p(\boldsymbol{\theta}\mid \vartheta,Y)
=
\frac{1}{Z(\vartheta,\beta;Y)}
\exp\!\left\{-\beta F_\beta(\boldsymbol{\theta},\vartheta,Y)\right\}.
\]
\end{proposition}

\begin{proof}
Insert the identity
\[
1
=
\int_{\mathbb R^{n_\theta}}
\delta^{n_\theta}
\!\left(
\boldsymbol{\theta}(\phi,Y)-\boldsymbol{\theta}
\right)
\,d^{n_\theta}\theta
\]
into the encoder partition function
\[
Z(\vartheta,\beta;Y)
=
\int_{\mathbb R^{n_\phi}}
\exp\!\left\{-\beta L_{\mathrm{rec}}(\phi,\vartheta)\right\}
\,d\phi .
\]
This gives
\[
Z(\vartheta,\beta;Y)
=
\int_{\mathbb R^{n_\theta}}
d^{n_\theta}\theta
\int_{\mathbb R^{n_\phi}}
\exp\!\left\{-\beta L_{\mathrm{rec}}(\phi,\vartheta)\right\}
\delta^{n_\theta}
\!\left(
\boldsymbol{\theta}(\phi,Y)-\boldsymbol{\theta}
\right)
\,d\phi .
\]
Define
\[
\exp\!\left\{-\beta F_\beta(\boldsymbol{\theta},\vartheta,Y)\right\}
:=
\int_{\mathbb R^{n_\phi}}
\exp\!\left\{-\beta L_{\mathrm{rec}}(\phi,\vartheta)\right\}
\delta^{n_\theta}
\!\left(
\boldsymbol{\theta}(\phi,Y)-\boldsymbol{\theta}
\right)
\,d\phi .
\]
Then
\[
Z(\vartheta,\beta;Y)
=
\int_{\mathbb R^{n_\theta}}
\exp\!\left\{-\beta F_\beta(\boldsymbol{\theta},\vartheta,Y)\right\}
\,d^{n_\theta}\theta ,
\]
and hence the induced density over collective variables is
\[
p(\boldsymbol{\theta}\mid \vartheta,Y)
=
\frac{
\exp\!\left\{-\beta F_\beta(\boldsymbol{\theta},\vartheta,Y)\right\}
}{
Z(\vartheta,\beta;Y)
}.
\]

Next, factor the restricted integral into an iso-$\boldsymbol{\theta}$ average
and a volume term:
\[
\exp\!\left\{-\beta F_\beta(\boldsymbol{\theta},\vartheta,Y)\right\}
=
\Omega(\boldsymbol{\theta},Y)
\left\langle
\exp\!\left\{-\beta L_{\mathrm{rec}}(\phi,\vartheta)\right\}
\right\rangle_{\boldsymbol{\theta}} .
\]
Taking $-\beta^{-1}\log$ of both sides gives
\[
F_\beta(\boldsymbol{\theta},\vartheta,Y)
=
-\frac{1}{\beta}
\log
\left\langle
\exp\!\left\{-\beta L_{\mathrm{rec}}(\phi,\vartheta)\right\}
\right\rangle_{\boldsymbol{\theta}}
-
\frac{1}{\beta}
\log \Omega(\boldsymbol{\theta},Y).
\]
Using $S(\boldsymbol{\theta},Y)=\log\Omega(\boldsymbol{\theta},Y)$ completes the
proof.
\end{proof}

\begin{corollary}[First-cumulant approximation]
\label{cor:first_cumulant_cv_free_energy}
Under the conditions of Proposition~\ref{prop:cv_free_energy}, choose collective
variables $\boldsymbol{\theta}=\boldsymbol{\theta}(\phi,Y)$ such that
fluctuations of $L_{\mathrm{rec}}(\phi,\vartheta)$ over each
iso-$\boldsymbol{\theta}$ set are small. Then conditional fluctuation terms
beyond the mean may be neglected, giving
\[
F_\beta(\boldsymbol{\theta},\vartheta,Y)
\approx
\left\langle
L_{\mathrm{rec}}(\phi,\vartheta)
\right\rangle_{\boldsymbol{\theta}}
-
\frac{1}{\beta}
S(\boldsymbol{\theta},Y).
\]
\end{corollary}

\begin{proof}
From Proposition~\ref{prop:cv_free_energy}, the exact collective-variable free
energy is
\[
F_\beta(\boldsymbol{\theta},\vartheta,Y)
=
-\frac{1}{\beta}
\log
\left\langle
\exp\!\left\{-\beta L_{\mathrm{rec}}(\phi,\vartheta)\right\}
\right\rangle_{\boldsymbol{\theta}}
-
\frac{1}{\beta}S(\boldsymbol{\theta},Y).
\]
For fixed $\boldsymbol{\theta}$, define the conditional mean reconstruction loss
\[
\overline L_{\boldsymbol{\theta}}(\vartheta)
:=
\left\langle
L_{\mathrm{rec}}(\phi,\vartheta)
\right\rangle_{\boldsymbol{\theta}},
\]
and write the reconstruction loss as its conditional mean plus a fluctuation:
\[
L_{\mathrm{rec}}(\phi,\vartheta)
=
\overline L_{\boldsymbol{\theta}}(\vartheta)
+
\Delta L_{\boldsymbol{\theta}}(\phi,\vartheta),
\]
where
\[
\Delta L_{\boldsymbol{\theta}}(\phi,\vartheta)
:=
L_{\mathrm{rec}}(\phi,\vartheta)
-
\overline L_{\boldsymbol{\theta}}(\vartheta).
\]
By construction,
\[
\left\langle
\Delta L_{\boldsymbol{\theta}}(\phi,\vartheta)
\right\rangle_{\boldsymbol{\theta}}
=
0.
\]
Substituting this decomposition into the exponential average gives
\[
\left\langle
\exp\!\left\{-\beta L_{\mathrm{rec}}(\phi,\vartheta)\right\}
\right\rangle_{\boldsymbol{\theta}}
=
\left\langle
\exp\!\left\{
-\beta
\left[
\overline L_{\boldsymbol{\theta}}(\vartheta)
+
\Delta L_{\boldsymbol{\theta}}(\phi,\vartheta)
\right]
\right\}
\right\rangle_{\boldsymbol{\theta}}.
\]
Since $\overline L_{\boldsymbol{\theta}}(\vartheta)$ is constant over the
iso-$\boldsymbol{\theta}$ set, it factors out:
\[
\left\langle
\exp\!\left\{-\beta L_{\mathrm{rec}}(\phi,\vartheta)\right\}
\right\rangle_{\boldsymbol{\theta}}
=
\exp\!\left\{
-\beta \overline L_{\boldsymbol{\theta}}(\vartheta)
\right\}
\left\langle
\exp\!\left\{
-\beta \Delta L_{\boldsymbol{\theta}}(\phi,\vartheta)
\right\}
\right\rangle_{\boldsymbol{\theta}}.
\]
Taking the logarithm gives
\[
\log
\left\langle
\exp\!\left\{-\beta L_{\mathrm{rec}}(\phi,\vartheta)\right\}
\right\rangle_{\boldsymbol{\theta}}
=
-\beta \overline L_{\boldsymbol{\theta}}(\vartheta)
+
\log
\left\langle
\exp\!\left\{
-\beta \Delta L_{\boldsymbol{\theta}}(\phi,\vartheta)
\right\}
\right\rangle_{\boldsymbol{\theta}}.
\]
Therefore,
\[
-\frac{1}{\beta}
\log
\left\langle
\exp\!\left\{-\beta L_{\mathrm{rec}}(\phi,\vartheta)\right\}
\right\rangle_{\boldsymbol{\theta}}
=
\overline L_{\boldsymbol{\theta}}(\vartheta)
-
\frac{1}{\beta}
\log
\left\langle
\exp\!\left\{
-\beta \Delta L_{\boldsymbol{\theta}}(\phi,\vartheta)
\right\}
\right\rangle_{\boldsymbol{\theta}}.
\]

We now expand the remaining fluctuation term. Since
$\langle \Delta L_{\boldsymbol{\theta}}\rangle_{\boldsymbol{\theta}}=0$,
\[
\left\langle
\exp\!\left\{
-\beta \Delta L_{\boldsymbol{\theta}}
\right\}
\right\rangle_{\boldsymbol{\theta}}
=
1
+
\frac{\beta^2}{2}
\left\langle
(\Delta L_{\boldsymbol{\theta}})^2
\right\rangle_{\boldsymbol{\theta}}
-
\frac{\beta^3}{6}
\left\langle
(\Delta L_{\boldsymbol{\theta}})^3
\right\rangle_{\boldsymbol{\theta}}
+
O(\beta^4).
\]
Taking the logarithm yields
\[
\log
\left\langle
\exp\!\left\{
-\beta \Delta L_{\boldsymbol{\theta}}
\right\}
\right\rangle_{\boldsymbol{\theta}}
=
\frac{\beta^2}{2}
\left\langle
(\Delta L_{\boldsymbol{\theta}})^2
\right\rangle_{\boldsymbol{\theta}}
-
\frac{\beta^3}{6}
\left\langle
(\Delta L_{\boldsymbol{\theta}})^3
\right\rangle_{\boldsymbol{\theta}}
+
O(\beta^4).
\]
Hence
\[
-\frac{1}{\beta}
\log
\left\langle
\exp\!\left\{-\beta L_{\mathrm{rec}}(\phi,\vartheta)\right\}
\right\rangle_{\boldsymbol{\theta}}
=
\overline L_{\boldsymbol{\theta}}(\vartheta)
-
\frac{\beta}{2}
\left\langle
(\Delta L_{\boldsymbol{\theta}})^2
\right\rangle_{\boldsymbol{\theta}}
+
\frac{\beta^2}{6}
\left\langle
(\Delta L_{\boldsymbol{\theta}})^3
\right\rangle_{\boldsymbol{\theta}}
+
O(\beta^3).
\]
Thus, the exact energy-like term contains the conditional mean reconstruction
loss plus corrections from conditional fluctuations of the reconstruction loss on
the iso-$\boldsymbol{\theta}$ set. If these fluctuations are small, or if these
higher-order fluctuation terms are neglected, then
\[
-\frac{1}{\beta}
\log
\left\langle
\exp\!\left\{-\beta L_{\mathrm{rec}}(\phi,\vartheta)\right\}
\right\rangle_{\boldsymbol{\theta}}
\approx
\overline L_{\boldsymbol{\theta}}(\vartheta)
=
\left\langle
L_{\mathrm{rec}}(\phi,\vartheta)
\right\rangle_{\boldsymbol{\theta}}.
\]
Substituting this approximation into the exact free-energy expression gives
\[
F_\beta(\boldsymbol{\theta},\vartheta,Y)
\approx
\left\langle
L_{\mathrm{rec}}(\phi,\vartheta)
\right\rangle_{\boldsymbol{\theta}}
-
\frac{1}{\beta}
S(\boldsymbol{\theta},Y),
\]
which proves the claim.
\end{proof}
Equivalently, the induced marginal density over collective variables can be
written in a Bayes-like form,
\[
p(\boldsymbol{\theta}\mid \vartheta,Y)
\propto
\exp\!\left\{
-\beta
\bar L_{\mathrm{rec}}(\boldsymbol{\theta},\vartheta,Y)
\right\}
\,
\Omega(\boldsymbol{\theta},Y),
\]
where
\[
\bar L_{\mathrm{rec}}(\boldsymbol{\theta},\vartheta,Y)
:=
\left\langle
L_{\mathrm{rec}}(\phi,\vartheta)
\right\rangle_{\boldsymbol{\theta}} .
\]
Thus $\Omega(\boldsymbol{\theta},Y)$ acts as an implicit entropic prior over
collective variables: collective-variable values supported by larger
encoder-parameter volume receive greater prior weight, without specifying an
explicit prior distribution.

\paragraph{Effect of decoder updates.}
For a fixed decoder-independent collective-variable map
$\boldsymbol{\theta}(\phi,Y)$, the entropy term depends only on the geometry of
the encoder-to-collective-variable map:
\[
S(\boldsymbol{\theta},Y)
=
\log
\int_{\mathbb R^{n_\phi}}
\delta^{n_\theta}
\!\left(
\boldsymbol{\theta}(\phi,Y)-\boldsymbol{\theta}
\right)
\,d\phi .
\]
Hence
\[
\nabla_\vartheta S(\boldsymbol{\theta},Y)=0.
\]
Therefore, decoder updates reshape the collective-variable free energy only
through the reconstruction term. Under the first-cumulant approximation, for an
infinitesimal decoder change
$\vartheta_{k+1}=\vartheta_k+\delta\vartheta$,
\[
\delta F_\beta(\boldsymbol{\theta})
:=
F_\beta(\boldsymbol{\theta},\vartheta_{k+1},Y)
-
F_\beta(\boldsymbol{\theta},\vartheta_k,Y)
\]
satisfies
\[
\delta F_\beta(\boldsymbol{\theta})
\approx
\left\langle
L_{\mathrm{rec}}(\phi,\vartheta_{k+1})
-
L_{\mathrm{rec}}(\phi,\vartheta_k)
\right\rangle_{\boldsymbol{\theta}} .
\]
Linearizing in $\delta\vartheta$ gives
\[
\delta F_\beta(\boldsymbol{\theta})
\approx
\left\langle
\nabla_\vartheta L_{\mathrm{rec}}(\phi,\vartheta_k)
\right\rangle_{\boldsymbol{\theta}}
\cdot
\delta\vartheta .
\]
Hence, in the collective-variable description, a decoder update changes the
induced probability of collective variables by altering their conditional
reconstruction energy, while the entropic preference for high-volume encoder
features remains fixed. This is the sense in which EAE performs entropically biased selection: although
$\boldsymbol{\theta}$ is not explicitly tracked, resampling from the updated
encoder Gibbs measure favors high-volume features that remain useful for
reconstruction, rather than latent directions the decoder can ignore.

\subsection{Decoder free-energy gradient and entropic selection\label{app:decoder_free_energy_gradient}}

Here we derive the decoder free-energy identity used in
Section~\ref{subsec:decoder_entropic_selection}. For fixed inverse temperature
$\beta>0$, define the encoder partition function
\begin{equation}
Z(\vartheta,\beta;Y)
=
\int_{\mathbb R^{n_\phi}}
\exp\!\left\{-\beta L_{\mathrm{rec}}(\phi,\vartheta)\right\}\,d\phi ,
\label{eq:app_decoder_partition}
\end{equation}
and the corresponding decoder free energy
\begin{equation}
F_{\mathrm{dec}}(\vartheta;Y)
=
-\frac{1}{\beta}\log Z(\vartheta,\beta;Y).
\label{eq:app_decoder_free_energy}
\end{equation}
This free energy is low when a large volume of encoder parameters achieves low
reconstruction loss under the decoder $D_\vartheta$.

\begin{proposition}[Decoder free-energy gradient]
\label{prop:decoder_free_energy_gradient}
Assume that $L_{\mathrm{rec}}(\phi,\vartheta)$ is differentiable in
$\vartheta$ and that differentiation under the integral sign is valid. Then
\begin{equation}
\nabla_\vartheta F_{\mathrm{dec}}(\vartheta;Y)
=
\mathbb E_{p_\vartheta(\phi\mid Y)}
\!\left[
\nabla_\vartheta L_{\mathrm{rec}}(\phi,\vartheta)
\right],
\label{eq:app_decoder_grad_phi}
\end{equation}
where
\begin{equation}
p_\vartheta(\phi\mid Y)
=
\frac{1}{Z(\vartheta,\beta;Y)}
\exp\!\left\{-\beta L_{\mathrm{rec}}(\phi,\vartheta)\right\}
\label{eq:app_encoder_gibbs}
\end{equation}
is the encoder Gibbs measure.
\end{proposition}

\begin{proof}
Starting from Eq.~\eqref{eq:app_decoder_free_energy},
\begin{align}
\nabla_\vartheta F_{\mathrm{dec}}(\vartheta;Y)
&=
-\frac{1}{\beta}
\frac{1}{Z(\vartheta,\beta;Y)}
\nabla_\vartheta Z(\vartheta,\beta;Y)
\nonumber\\
&=
-\frac{1}{\beta}
\frac{1}{Z(\vartheta,\beta;Y)}
\int_{\mathbb R^{n_\phi}}
\nabla_\vartheta
\exp\!\left\{-\beta L_{\mathrm{rec}}(\phi,\vartheta)\right\}
\,d\phi
\nonumber\\
&=
\frac{1}{Z(\vartheta,\beta;Y)}
\int_{\mathbb R^{n_\phi}}
\nabla_\vartheta L_{\mathrm{rec}}(\phi,\vartheta)
\exp\!\left\{-\beta L_{\mathrm{rec}}(\phi,\vartheta)\right\}
\,d\phi .
\end{align}
Using the definition of $p_\vartheta(\phi\mid Y)$ in
Eq.~\eqref{eq:app_encoder_gibbs} gives Eq.~\eqref{eq:app_decoder_grad_phi}.
\end{proof}

\paragraph{Entropic-selection interpretation.}
The partition function in Eq.~\eqref{eq:app_decoder_partition} integrates
$\exp\{-\beta L_{\mathrm{rec}}(\phi,\vartheta)\}$ over encoder-parameter space.
Therefore, $Z(\vartheta,\beta;Y)$ is large when the decoder admits a broad
family of encoders with low reconstruction loss. Minimizing
$F_{\mathrm{dec}}(\vartheta;Y)$ consequently favors decoders whose reconstruction
performance is supported by high-volume regions of encoder space.

\subsection{Simmering sampler for the encoder ensemble\label{app:eae_simmering}}

For fixed decoder parameters $\vartheta_k$, EAE requires samples from the encoder
Gibbs measure
\[
p_{\vartheta_k}(\phi\mid Y)
\propto
\exp\!\left\{-\beta L_{\mathrm{rec}}(\phi,\vartheta_k)\right\}.
\]
We obtain these samples using Simmering~\citep{babayanSufficient2025}, a
finite-temperature molecular-dynamics sampler. In this construction, encoder
parameters $\phi$ are treated as particle positions, and the reconstruction loss
$L_{\mathrm{rec}}(\phi,\vartheta_k)$ acts as the potential energy. The force on
the encoder parameters is therefore
\[
-\nabla_\phi L_{\mathrm{rec}}(\phi,\vartheta_k),
\]
which is computed by backpropagation. Rather than descending to a single
optimizer-selected encoder, Simmering couples the encoder dynamics to a
thermostat so that the trajectory samples an approximate canonical ensemble of
near-optimal encoders.

\paragraph{Augmented phase space.}
We introduce auxiliary momenta $r_\phi\in\mathbb R^{n_\phi}$ and a positive mass
matrix $M_\phi$. Up to a normalization constant, the encoder partition function
can be lifted to phase space as
\[
Z(\vartheta_k,\beta;Y)
\propto
\int d\phi\,dr_\phi\,
\exp\!\left\{
-\beta \mathcal H(\phi,r_\phi;\vartheta_k,Y)
\right\},
\]
with Hamiltonian
\[
\mathcal H(\phi,r_\phi;\vartheta_k,Y)
=
L_{\mathrm{rec}}(\phi,\vartheta_k)
+
\frac{1}{2}r_\phi^\top M_\phi^{-1}r_\phi .
\]
Integrating out $r_\phi$ recovers the original Gibbs measure over $\phi$, so the
auxiliary momenta do not change the target encoder distribution; they only
provide dynamics for sampling it.

\paragraph{Sampling.}
In practice, Simmering uses a Nos\'e--Hoover chain thermostat and a symplectic
integration scheme to maintain the target temperature and sample the canonical
parameter ensemble. We do not repeat the discretized integration algorithm here;
the full sampler, including equilibration and numerical integration details, is
given in \citet{babayanSufficient2025}. The implementation choices needed to
reproduce our EAE experiments are reported in Appendix~\ref{app:exp_setup}.

During sampling, encoder configurations $\{\phi^{(i)}\}_{i=1}^M$ are collected
along the trajectory and mapped through the encoder to obtain latent codes
\[
z_n^{(i)}=E_{\phi^{(i)}}(y_n),
\qquad
i=1,\ldots,M .
\]
The resulting empirical encoder ensemble provides the decoder-gradient samples
used in Algorithm~\ref{alg:eae_training}.

\subsection{Sampling and generation from a trained EAE}
\label{app:eae_sampling}
\begin{algorithm}[H]
\caption{Sampling from a trained Entropic Autoencoder}
\label{alg:eae_sampling}
\begin{algorithmic}[1]
\Require data or query set $Y_{\mathrm{eval}}=\{y_n\}_{n=1}^{N_{\mathrm{eval}}}$
\Require trained decoder parameters $\vartheta^\star$
\Require encoder ensemble $\mathcal E=\{\phi^{(i)}\}_{i=1}^M$
\State $\mathcal Z\gets\emptyset$
\For{$i=1,\ldots,M$}
    \State Fix encoder parameters $\phi^{(i)}\in\mathcal E$
    \For{$n=1,\ldots,N_{\mathrm{eval}}$}
        \State Compute latent sample $z_n^{(i)}\gets E_{\phi^{(i)}}(y_n)$
        \State Store $\mathcal Z\gets\mathcal Z\cup\{z_n^{(i)}\}$
    \EndFor
\EndFor
\State \Return latent ensemble $\mathcal Z$
\end{algorithmic}
\end{algorithm}

\subsection{Experiment Setup}
\label{app:exp_setup}
All experiments were performed using the TensorFlow library\citep{tf} (released under an Apache 2.0 license)  in Python, using one Nvidia GeForce RTX 4060 GPU. Unless otherwise indicated, the standard TensorFlow implementation and default parameters were used for a given function or algorithm. For sampling the canonical ensemble, we used the Simmering implementation by Ref. \citep{babayanSufficient2025} published on Zenodo \citep{SimmeringCode} (released under a Creative Commons 4.0 Attribution Internal license). Any hyperparameters not mentioned below can be assumed to be set to the default values of the open-source Simmering implementation.

\subsubsection{Setup: Recovery of meaningful low-dimensional representations with an EAE \label{sec:appendix_exp1}}

This experimental setup description pertains to experiments presented in Section \ref{sec:exp_1}. The problem setup of Section \ref{sec:exp_1} is based on the work of Ref. \citep{championDatadriven2019}. The modelling task aims to reproduce the time evolution of a lambda-omega reaction-diffusion system, in which a spiral wave forms. Lambda-omega reaction-diffusion systems are a model for oscillation dynamics near a Hopf bifurcation\citep{murrayMathBio1990}, and thus can be described in a lower-dimensional form characterizing the ``orbit'' around the limit cycle. This problem setup, with a known lower dimensional form, allows us to assess the meaningfulness of the latent-space representations identified by an EAE. The data consists of time snapshots of a wave function and its time derivative in two spatial dimensions ($100\times 100$ dimensional time snapshots for both the wave and its time derivative), and was generated via an open-source numerical integration code provided by the authors of Ref.\citep{championDatadriven2019} on GitHub. Three data samples, at timesteps $0,5000,10000$ are shown in Fig. \ref{fig:toy_model}a. 4000 constant-time snapshots, randomly sampled from the first 9000 time snapshots of the original dataset, were used for training, and another distinct 1000 samples, taken randomly from the first 9000 time snapshots, were used as the validation set. The final 1000 time snapshots were used as the test set.

The goal of training the EAE is to generate accurate latent variable samples that inform the identification of basis function coefficients in a differential, i.e., governing, equation describing the lower-dimensional dynamics. As in Ref. \citep{championDatadriven2019}, we model the latent variable time dynamics as
\begin{equation}
    \dot{\mathbf{z}}=\Theta(\mathbf{z})\Xi, \label{eq:sindy_latent_dyn}
\end{equation} 
where $\Theta(\mathbf{z})=[f_1(\mathbf{z})\dots  f_p(\mathbf{z})]$ is a matrix containing each of the candidate basis functions $\{ f_i\}$ evaluated at $\mathbf{z}$, and $\Xi$ is a matrix of unknown coefficients for each candidate basis function and each latent dimension, e.g., if there are 12 basis function candidates and 2 latent dimensions, $\Xi$ is a $12\times2$ matrix. We use the same set of basis functions as the corresponding reaction-diffusion example in \citep{championDatadriven2019} but remove the constant term as we do not specify the initial conditions for the latent variables in the objective function. The objective function is also based on the SINDy autoencoder objective function in \citep{championDatadriven2019} but with two key differences: we remove the ``SINDy regularization'' term (applying an $L_1$ norm regularization on basis coefficients), and we sample $\Xi$ over the encoder ensemble, rather using than optimizing $\Xi$ in tandem with the autoencoder parameters. These changes transform the original problem setup into one that trains only the autoencoder parameters with an objective only containing reconstruction loss terms, in line with the EAE training framework.

Our loss function is thus given by
\begin{equation}
    \mathcal{L} = \lambda_{1}\; \| \mathbf{x} - \mathbf{\hat{x}} \|_2^2 + \lambda_{2} \left\| \dot{\mathbf{x}} - \hat{\dot{\mathbf{x}}} \right\|_2^2, \label{eq:sindy_loss}
\end{equation}
where $\mathbf{x},\mathbf{\dot{x}}$ denote the two quantities to be reconstructed (the wave function and its time derivative, denoted henceforth with a dot), and $\lambda_{1},\lambda_{2}$ are hyperparameters scale the relative magnitudes of the two reconstruction terms. 
We compute $\hat{\dot{x}}$ in Eq. \ref{eq:sindy_loss} via Eq. \ref{eq:sindy_latent_dyn} and chain rule
\begin{equation}
    \mathbf{\hat{\dot{x}}}=\frac{d\mathbf{\hat{x}}}{dt}=\frac{d\mathbf{\hat{x}}}{d\mathbf{\hat{z}}}\frac{d\mathbf{\hat{z}}}{dt}=\frac{d\mathbf{\hat{x}}}{d\mathbf{\hat{z}}}\Theta(\mathbf{\hat{z}})\Xi, \label{eq:sindy_dxdt}
\end{equation}
where $\mathbf{\hat{x}},\mathbf{\hat{z}}$ denote the reconstruction of $\mathbf{x}$ and the encoding of $\mathbf{x}$ respectively. We sample $\Xi$ using $\mathbf{\hat{z}}$ by solving for $\Xi$ using the ``Sindy loss in $\dot{z}$'' term in the original problem formulation,
\begin{equation}
    \Xi = (\Theta(\mathbf{\hat{z}})^T\Theta(\mathbf{\hat{z}}))^{-1}\Theta(\mathbf{\hat{z}})^T\frac{d\mathbf{\hat{z}}}{d\mathbf{x}}\frac{d\mathbf{x}}{dt}, \label{eq:sindy_xi_def}
\end{equation}
where $\frac{d\mathbf{x}}{dt}=\mathbf{\dot{x}}$ is the known data describing the wave's time derivative. In each iteration of training, the reconstruction loss computation consists of (1) encoding the latent variable $\mathbf{\hat{z}}$ and predicting the reconstruction $\mathbf{\hat{x}}$ using the autoencoder, (2) solving for $\Xi$ using Eq. \ref{eq:sindy_xi_def}, (3) solving for $\mathbf{\dot{\hat{x}}}$ using Eq. \ref{eq:sindy_dxdt}, and (4) evaluating the loss (Eq. \ref{eq:sindy_loss}).

The autoencoder architecture consisted of an encoder with one hidden layer with 256 hidden units and an exponential linear unit activation, and a linear output layer producing 2-dimensional latent variables. The decoder has a hidden layer with 256 units and an exponential linear unit, and an output layer expanding the outputs back to the input dimension ($10^4$). The model was trained for 100 decoder update iterations, with 200 encoder update iterations per decoder update. The decoder gradients were applied with the Adam optimizer, with a learning rate of $10^{-3}$. The loss weights in the reconstruction loss (Eq. \ref{eq:sindy_loss}) were set as $\lambda_1=1.0, \lambda_2=20.0$. The thermostat parameters used for encoder ensemble sampling were: thermostat initial and final temperature $T=10^{-4}$, real particle mass $m=10^{-6}$, learning rate/time step size $10^{-5}$. To improve convergence speed, based on Ref.\citep{hmc}, velocities were initialized via sampling from a Gaussian distribution with a mean of $0$ and a variance $T/m$, where $T$ and $m$ are the thermostat temperature and real particle mass respectively, and reset by resampling from this Gaussian distribution every 5 encoder iterations. 

After optimizing the decoder parameters, 5000 consecutive encoder ensemble samples were collected with the decoder parameters fixed, and with the velocity randomization still in effect. These sampled encoder parameters were used to estimate $\Xi$ via Eq. \ref{eq:sindy_xi_def} on test data. We computed the expectation of each estimated coefficient, and considered any terms with a magnitude above $0.1$ as significant (in accordance with the original work, Ref. \citep{championDatadriven2019}), presented in Table \ref{tab:appendix_coeffs}. For complex, multimodal distributions, expectation values are not the most descriptive measures, so we also included the dominant mode (determined by kernel density estimation usig the \texttt{scipy.stats.gaussian\textunderscore kde} library). Subsequent analysis only concerns basis coefficients for which both the expectation and mode were above the threshold. The only coefficients above the significance threshold (highlighted in Table \ref{tab:appendix_coeffs}) were the linear ($z_1$, and $z_2$ terms) and the sine terms ($\sin z_1$, $\sin z_2$). Appropriate combinations of these basis functions correspond to linear or non-linear oscillation dynamics. Analysis (Fig. \ref{fig:appendix_coeff_corr}) of coefficient correlations show that, beyond displaying the expected strong correlation between coefficient combinations corresponding to descriptions of linear and non-linear ellipsoid orbits, the pattern of correlation between basis coefficients shows that the model exhibits coordinate system (the relative roles of $z_1$, $z_2$ in dynamics) ambiguity. We chose to limit our visualization of this coordinate ambiguity, i.e., superposition of admissible dynamics, in Fig. \ref{fig:toy_model}d to the strongly correlated linear basis coefficient terms for simplicity. We fit the most co-occurring coefficient combinations identified in Fig. \ref{fig:toy_model}d with kernel density estimation to determine each pair's corresponding equations of motion. We then used our trained autoencoder to generate an ``initial condition'' for the latent variables $z_1,z_2$ by computing the encoding of the earliest-time test data sample. We then numerically integrated both identified equations of motion (with \texttt{scipy.integrate.odeint}) with the this initial condition to generate a visualization of the limit cycle orbit (Fig. \ref{fig:toy_model}e--f).   

\begin{figure}
    \centering
    \begin{minipage}[b]{0.45\textwidth}
    \centering
    \captionof{table}{Basis coefficients learned by the EAE corresponding to the training problem described in Appendix \ref{sec:appendix_exp1}. For each coefficient, both the average and the mode of the distributions learned by sampling over the ensemble of encoders after training are presented to account for potential multi-modality. Coefficients above the significance threshold (0.1) are bolded. An analysis of the relationship between the linear coefficients $z_1$,$z_2$ is presented in Fig. \ref{fig:toy_model}. }
    \begin{tabular}{|c|c|c|c|c|}
    \hline
     &   \multicolumn{2}{c}{$\dot{z}_1$ coefficient} & \multicolumn{2}{|c|}{$\dot{z}_2$ coefficient} \\
     \cline{2-5}
     &  Avg. & Mode & Avg. & Mode \\
     \hline
     \bm{$z_1$} & 0.47 & 0.48 & -1.98 & -1.98 \\
     \bm{$z_2$} & 1.00 & 1.17 & -1.37 & -0.38 \\
     $z^2_1$ & 0.00 & 0.00 & 0.01 & 0.00 \\
     $z_1z_2$ & 0.00  & 0.00 & -0.03 & -0.03\\
     $z^2_2$ & -0.03 & -0.01 & -0.03 & -0.01 \\
     $z^3_1 $ & 0.00 & -0.01 & -0.05 & -0.07  \\
     $z^2_1z_2$ & -0.02 & -0.03 & 0.09 & -0.03  \\
     $z_1z^2_2$ & -0.03 & -0.03 & 0.14 & -0.03  \\
     $z^3_2$& 0.03 & 0.02 & 0.09 & 0.02 \\
     \bm{$\mathbf{\sin}(z_1)$} & 0.36 & 0.31 & 0.13 &  0.31  \\
     \bm{$\mathbf{\sin}(z_2)$} & 0.86 & 0.81  & 0.11 & 0.85 \\
     \hline
    \end{tabular}
    \label{tab:appendix_coeffs}
    \end{minipage}
    \hfill
    \begin{minipage}[b]{0.45\textwidth}
    \centering
    \includegraphics[width=\linewidth]{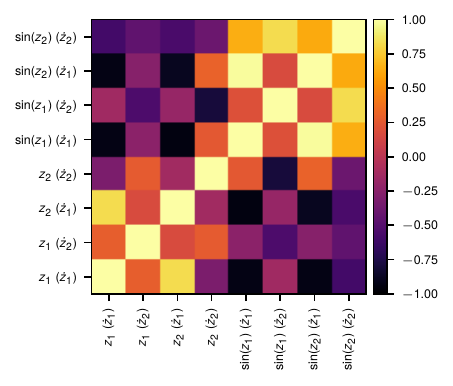}
    \caption{Corresponding to coefficient distributions learned by the EAE, and tabulated in Table \ref{tab:appendix_coeffs}, we compare the Pearson product-moment correlation coefficients (magnitude reflected by colorbar) between the most significant basis coefficients (bolded in Table \ref{tab:appendix_coeffs}). Correlations between corresponding linear and non-linear terms, as well as between different linear terms, suggest  a superposition of learned dynamics. }
    \label{fig:appendix_coeff_corr}
    \end{minipage}
    
    \caption*{Analysis of basis coefficients learned by the EAE, whose training setup is described in Appendix \ref{sec:appendix_exp1}.}
    \label{fig:toy_model_combo}
\end{figure}

\subsubsection{Setup: Comparison of latent variable distributions in image reconstruction tasks
\label{sec:appendix_exp2}}

This experimental setup description pertains to experiments presented in Section \ref{sec:exp_2}.

In both problem settings, the VAE was trained via marginal log-likelihood evidence lower bound (ELBO) maximization by optimizing the single sample Monte Carlo estimate of the standard VAE loss, i.e., where the strength of prior matching $\lambda=1$.\citep{higginsBetaVAE2017} In this standard VAE training protocol, each latent dimension's a priori distribution is assumed to be a standard normal distribution, and the approximate latent variable posterior distributions are also Gaussian. The reconstruction error component of the ELBO is computed via binary crossentropy between the true pixel values and the decoder outputs (with a sigmoid applied on the latter). 
To quantify reconstruction capability, we calculated the mean-squared error between test images and their reconstructions. To quantitatively assess the extent of posterior collapse, we calculated the number of ``active latent dimensions'' by counting the number of latent dimensions with a latent variable variance (variance of each latent dimension's mean for the VAE, and latent variable variance otherwise) above a threshold of 0.01 \citep{burdaImportance2015}.

For all MNIST experiments, the standard train/test split was used (60,000 training samples, 10,000 testing samples), with 10,000 samples reserved from the training set for validation. The data was sourced from the TensorFlow datasets repository, but originates from Ref. \citep{lecunMNIST2010} (Creative Commons Attribution-Share Alike 3.0 license). All image pixels were rescaled to the range $[0,1]$ before training. All models had a latent dimension of 64 units, and the same decoder architecture: $[64, 784]$ units, with a ReLU activation function on the non-output layer. The encoder architecture was almost identical for all three setups tested ($[256, 128, 128, 64/128]$ units, and all non-output layers used a ReLU activation function), except for the VAE latent dimensionality being double that of the the vanilla autoencoder and EAE, as it learns a mean and a log-variance per latent dimension rather than a single value per latent dimension. All optimized model components (vanilla autoencoder and VAE parameters, and EAE decoder) were trained with TensorFlow's Adam optimizer implementation, with a learning rate of 0.001, and a batch size of 32. The VAE was trained for 70 epochs, and the vanilla AE was trained for 140 epochs. The EAE was trained for 250 decoder iterations, with 2 epochs (3126 iterations) of encoder ensemble-induced free-energy gradient samples per decoder parameter update. For canonical ensemble sampling, we used the following thermostat hyperparameters: learning rate of $10^{-5}$, thermostat temperature $T=10^{-8}$, real particle mass $m=10^{-6}$, and Nos\'e-Hoover chain length of 30.

To generate the linear interpolation visualization shown in Fig. \ref{fig:mnist_compar}a, we first computed each model's average encoding across all test images corresponding to a particular label, i.e., digit $\bar{\mathbf{z}}_{\text{digit}}$. For the EAE, this average is taken over ensemble members per digit sample, then over all digit samples. For the VAE, the sampled reparameterization corresponding to a given image's mean and log-variance was averaged across all images of a particular digit. For the vanilla autoencoder, the encoder outputs were averaged across all images within a digit category. Then, to interpolate, we computed a linear interpolation between each model's digit encodings as
\begin{equation}
    \mathbf{z}_{\text{interpol.}}=\alpha \bar{\mathbf{z}}_{\text{digit1}} + (1-\alpha) \bar{\mathbf{z}}_{\text{digit2}},
\end{equation}
where $\alpha$ scales the interpolation between two digits. This interpolation $\mathbf{z}_{\text{interpol.}}$ was provided as an input to each model's decoder to generate the interpolated images. The average across digit images is provided for comparison (labelled ``Test Data'' in Fig. \ref{fig:mnist_compar}a).

To compute the latent variable distributions presented in Fig. \ref{fig:mnist_compar}b, we chose three random latent space dimensions that all had ``active'' latent units, and plotted the latent variable distribution (either the latent variables themselves, or the encoded mean for the VAE) across the test set, separated by digit category. The measured activity of the presented latent dimension is provided in the top right corner of each subplot, denoted $A_n$. The category-distinguishing latent distribution characteristic of the EAE was consistent across other randomly sampled dimensions that were active across all models, as shown in Fig. \ref{fig:appendix_mnist_vis}. 

\begin{figure}
    \centering
    \includegraphics[width=\linewidth]{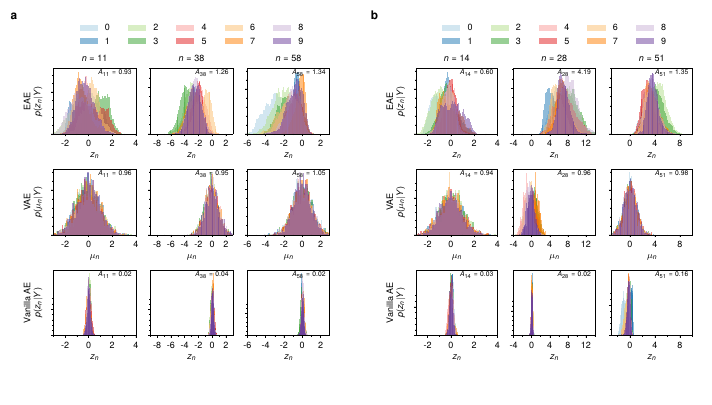}
    \caption{In addition to Fig. \ref{fig:mnist_compar}b, we include six more randomly chosen non-collapsed dimensions (panels a,b) from all three autoencoders to highlight the EAE's capacity to learn multimodal, as well as uniquely category-faithful, latent representations. In dimension $n=28$ (panel b), the VAE exhibits some non-overlapping distribution features, as does the vanilla AE for $n=51$ (panel b), but the distributions for these two models' latent variables still predominantly overlap for different digits.}
    \label{fig:appendix_mnist_vis}
\end{figure}

The Frey Faces dataset was sourced from Ref. \citep{FreyFace}. The data was originally provided to the database host by Dr. Brendan Frey for academic use, which aligns with the use case in this work. Pixels were mapped to the $[0,1]$ range before training. We used 80\% of the data for training, 10\% for validation and 10\% for testing. The data were partitioned randomly. We used the same neural network architecture as for the MNIST training setup, save for the encoder input and decoder output dimensions, which were condensed to $28\times20$ (from $28\times28$). As for MNIST, we used the Adam optimizer to apply the EAE decoder gradient updates, and to train the VAE and vanilla AE, with a learning rate of $10^{-3}$. All three models were trained without batching (since the full dataset fit on the GPU). The VAE was trained for $10^{5}$ epochs, and the vanilla AE was trained for $1.1\times 10^4$ epochs.

To train the EAE, we updated the decoder 1500 times, with 1000 encoder samples contributing to each decoder update. For sampling the ensemble of encoders during training, we used the same thermostat parameters as for the MNIST training problem, but with a thermostat temperature of $3\times10^{-8}$. After training, we collected 1000 consecutively sampled encoders with fixed, optimized decoder parameters for computing the latent variable activity.

We chose to omit interpolation and latent variable analysis for the Frey Faces dataset as the data are time-ordered samples of a person's face movement, and thus both lack straightforward underlying structure to study (such as category distinction) and simplify the task of interpolation (since the data are themselves provide examples of an interpolation of face movement).

\subsubsection{Setup: EAE Dynamics and Reconstruction Morphology \label{sec:appendix_exp3}}

This experimental setup description pertains to experiments presented in Section \ref{sec:exp_3}. 

For these experiments, we trained two EAEs on image reconstruction using the CelebA dataset \citep{celebA} (released for non-commercial research purposes only, which align with the use case in this work). For both EAEs, we used the standard train/val/test split, and for pre-processing, we center-cropped the images to a $64\times64\times3$ resolution, and rescaled the pixels to the $[0,1]$ range. For both training instances, we used an encoder architecture with 4 2D convolutional layers (32, 64, 128 and 256 filters in each layer), with each layer learning a $(4\times4)$ kernel, and using a stride of 2, ``same''-type padding (in TensorFlow), and a reLU activation function. After the four convolutional layers, the input was flattened and fed into a densely connected layer with 256 units and a linear activatio function, resulting in a 256-dimensional latent. This latent input was fed into a decoder with a 1024-unit densely connected layer (with a ReLU activation), after which the input was reshaped to dimensions $(4,4,64)$ and fed into a sequence of three transpose 2D convolutional layers with 64, 32 and 3 filters respectively. Each transpose 2D convolutional layer had a $4\times 4$ kernel, zero padding and a reLU activation function applied. The inner two transpose convolutional layers had strides of 4, whereas the outer layer had a stride of 1. We applied the estimated decoder gradients using the Adam optimizer with a learning rate of $10^{-3}$, and trained both the encoder and decoder with a batch size of 32.

To train the higher-temperature model, we updated the decoder parameters 700 times, with 5087 encoders sampled (each with a different data batch, over the course of one epoch) per decoder update. The thermostat parameters used for encoder sampling were: thermostat temperature $T=3\times10^{-8}$, real particle mass $m=10^{-6}$ and learning rate $10^{-5}$.

To train the lower-temperature model, we updated the decoder parameters 700 times, with 5087 encoders sampled (each with a different data batch, over the course of one epoch) per decoder update. The thermostat parameters used for encoder sampling were: thermostat temperature $T=1.5\times10^{-9}$, real particle mass $m=10^{-8}$ and learning rate $10^{-5}$.

To sample the images shown in Fig. \ref{fig:celebA}b--c, we sampled 1000 of the corresponding model's final decoder update iteration's encoder ensemble encodings for randomly chosen images from the test set, and used the optimized decoder to reconstruct the image from the ensemble-averaged encoding.

\subsection{Limitations, Broader Impacts and Future Work \label{app:limits_futurework}}

\subsubsection{Limitations \label{app:limits}}








The theoretical background of EAE relies on standard assumptions from
computational statistical mechanics, and in particular, molecular dynamics. For
fixed decoder parameters, we assume that the Simmering trajectory is sufficiently
equilibrated and ergodic so that finite-time samples provide an informative
approximation to the canonical ensemble of encoders. We also assume that the
integration step size is chosen so that the symplectic integrator accurately
approximates the continuous-time dynamics. The collective-variable interpretation
further assumes that the first-cumulant approximation captures the dominant
coarse-grained contribution to the free energy. Violations of these assumptions
may weaken the connection between the theoretical interpretation and practical
implementations, for example, when finite trajectories mix poorly, or higher-order fluctuations within iso-collective-variable sets are important.

The main hyperparameter that affects EAE training is the choice of temperature $T$. Since any particular dataset-parameterization combination forms a distinct loss landscape, each problem setup will exhibit different sensitivity to temperature. Thus, there is a training time cost associated with tuning this hyperparameter. Moreover, given a range of temperatures that produce ``sufficiently accurate'' EAEs, the unique role of temperature in training, where it modulates the extent to which reconstruction loss minimization is prioritized in the EAE, requires the user to make a subjective choice about what degree of accuracy is appropriate for the use case. This property of the training process introduces an additional model design decision that differs from other hyperparameter choices, but can ultimately allow for more control over model outcomes. 

The proposed training framework scales with both model dimensionality and ensemble encoder size. The backbone of our EAE implementation is the encoder ensemble sampling method (Simmering, described in Appendix \ref{app:eae_simmering}), which, according to Ref. \citep{babayanSufficient2025} scales similarly to a comparable first-order optimization algorithm, like Adam, with model size. The appropriate ensemble encoder size, i.e., how many encoder samples are used to update the decoder parameters, must be chosen based on achieved prediction variance, or outcome alignment with accuracy targets, and thus differs from problem to problem. However, serendipitous model architecture and temperature choices have resulted in faster convergence (in epochs) with Simmering (in its original implementation, in Ref. \citep{babayanSufficient2025}) than is achieveable with optimization alone, so it is possible that an analogous effect could be observed in EAE training given an appropriate model architecture, temperature and ensemble size choice. Otherwise, the EAE is as sensitive to dataset size or data feature dimensionality as other neural networks are: highly descriptive datasets benefit training convergence and accuracy, and high-dimensional data necessitates larger architectures that take longer to train. Overall, the EAE framework can incur additional computational overhead from the encoder ensemble sampling process, but we note that all experiments were generated on a consumer-grade computer and thus the proposed framework is not generally computationally expensive to the point of being impractical. 

High-dimensional data are inherent in dimensionality reduction tasks, and data batching is often employed to allow data to fit in GPU memory for GPU-accelerated training. Batching introduces an additional, anisotropic and problem-dependent noise source in the gradient, superimposed over any pre-existing data noise. The anisotropy of batching noise can complicate temperature selection by creating a ``noise floor'' in loss below which loss fluctuations (typically produced by the thermostat during encoder ensemble sampling) are instead batching-noise dominated. To facilitate straightforward temperature selection, batch sizes should be as large as possible (to reduce the presence of batching noise), subject to computer memory restrictions, or thermostat temperature ought to be increased if batch size is reduced.

\subsubsection{Broader Impact and Future Work \label{app:broader_impact}}

The broader impacts of improved generative model performance are multifaceted. First, it is prudent to acknowledge outright that improving the potential performance of any generative model risks facilitating more complex or harmful instances of misuse. However, in the particular case of this work, we believe the issue of posterior collapse as a consistent failure mode in VAEs is troubling and worth mitigating because VAEs underpin more sophisticated generative modelling methods such as diffusion models and are thus ubiquitous in the current consumer-accessible generative model landscape. From a fairness-sensitive perspective, posterior collapse can exacerbate bias propagation in models, since the model can only capture the most dominant trends in the data, disproportionately omitting features assocaited with underrepresented data populations. In contrast, the more comprehensive data representations and flexible insight into learned latent structure offered by the EAE can be used to identify less-dominant modes in generative behaviour and guide debiasing or fairness-promoting machine learning processes. This application of EAEs can mitigate the potential risks of enabling improved generative model performance and is the subject of current, ongoing work.

More generally, the method introduced in this work (Section \ref{sec:eae}), and associated experiments (Sections \ref{sec:exp_1}--\ref{sec:exp_3}) collectively present one implementation of an EAE training framework. Though we employed a particular thermostat (see Section \ref{app:eae_simmering}) for encoder ensemble sampling, thermostat-inspired approaches from other contexts, e.g., Langevin sampling, could offer performance benefits in EAE training that were not explored in this work. In addition, all experiments presented in this work employed a fixed-temperature training regime, but the success of annealing temperature-like parameters (like KL annealing) suggests that there may exist more efficient decoder parameter search strategies that employ a temperature schedule.

\end{document}